\documentclass[journal, draftcls, one column, 12pt]{IEEEtran}
\usepackage{cite}
\usepackage{algorithm}
\usepackage{url}
\usepackage{mathtools}
\usepackage{algorithmic}
\usepackage{booktabs}
\usepackage{amsmath,amssymb,amsfonts}
\usepackage{algorithmic}
\usepackage{graphicx}
\usepackage{textcomp}
\usepackage{xcolor}
\usepackage{braket}
\newtheorem{proposition}{Proposition}
\newtheorem{proof}{Proof}	
\newtheorem{definition}{Definition}
\def\BibTeX{{\rm B\kern-.05em{\sc i\kern-.025em b}\kern-.08em
		T\kern-.1667em\lower.7ex\hbox{E}\kern-.125emX}}
\begin{document}
	
	\title{ Estimating Kullback-Leibler Divergence Using Kernel Machines
	}
	
	\author{\IEEEauthorblockN{Kartik Ahuja} \\
		\IEEEauthorblockA{Electrical and Computer Engineering Department, UCLA  \\
			Email: ahujak@ucla.edu
	}}
	
	\maketitle
	
	\begin{abstract}
		Recently, a method called the Mutual Information Neural Estimator (MINE) that uses neural networks has been proposed to estimate mutual information and more generally the Kullback-Leibler (KL) divergence between two distributions. The method uses the Donsker-Varadhan representation to arrive at the estimate of the KL divergence and is better than the existing estimators in terms of scalability and flexibility. The output of MINE algorithm is not guaranteed to be a consistent estimator. We propose a new estimator that instead of searching among functions characterized by neural networks searches the functions in a Reproducing Kernel Hilbert Space. We prove that the proposed estimator is consistent. We carry out simulations and show that when the datasets are small the proposed estimator is more reliable than the MINE estimator and when the datasets are large the performance of the two methods are close.  
	\end{abstract}
	
	
	\section{Introduction}
	Kullback-Leibler (KL) divergence is one of the fundamental quantities in statistics and machine learning. It is used to measure the distance between two probability distributions. Mutual information, which is another fundamental quantity, is a special case of  KL divergence. It measures the information shared between two random variables and is equal to the  KL divergence between the joint and product distributions of the two random variables. It is used in several applications such as feature selection \cite{peng2005feature}, clustering \cite{romano2014standardized}, and representation learning  \cite{chen2016infogan}. Estimation of KL divergence and mutual information is a challenging task and developing estimators for these quantities continue to be an active area of research.
	
	Recently a method called Mutual Information Neural Estimation (MINE) \cite{belghazi2018mine} has been proposed to estimate the KL divergence between two distributions.  The key ideas in MINE are explained as follows:
	\begin{itemize}
		\item Use the Donkser-Varadhan (DV) \cite{donsker1983asymptotic} representation to express the KL divergence.
		\item Use a family of functions characterized by neural networks in the DV representation to build the estimator.
	\end{itemize}
	
	The authors in \cite{belghazi2018mine} used MINE to estimate the mutual information and showed that their estimator is better than the estimators in the literature \cite{kraskov2004estimating} \cite{perez2008kullback} in terms of the bias in many cases. MINE is a general purpose estimator as it estimates the KL divergence and not just mutual information. However, the  estimator constructed in \cite{belghazi2018mine} using the main algorithm is not guaranteed to be consistent (explained later). In this work, we propose a new estimator of KL divergence to address this issue.  We also rely on the Donsker-Varadhan representation to build our estimator. 
	If we estimate the KL divergence using DV representation, then we do not need to estimate the probability distributions directly unlike the standard estimators \cite{kraskov2004estimating}.
	Instead of searching in the space of neural network families (as in \cite{belghazi2018mine}) we set the search space as a Reproducing Kernel Hilbert Space (RKHS) and hence we name the estimator as the Kernel KL divergence estimator (KKLE). We are able to show that the search in RKHS reduces to solving a convex learning problem.  This enables us to prove that the estimator  we derive is consistent. 
	There are other estimators in the literature that are based on a different representation namely the f-divergence based representation \cite{nguyen2010estimating}. However, the DV representation, which is the focus of this work, is provably a tighter  representation than the f-divergence representation \cite{belghazi2018mine}. 
	

	In the experiments section, we compare the proposed KKLE with MINE estimator. We  carry out simulations over large datasets to show that the performances of both MINE and KKLE  are comparable. We also compare the two estimators for small datasets and we find that the KKLE  estimator is better than the MINE estimator. We also provide insights to explain why KKLE is expected to perform well. A python notebook to illustrate the experiments is available at \url{https://github.com/ahujak/KKLE}.

	\section{Problem Formulation and Approach}
	We first give a brief background. KL divergence is a quantity that is used to measure the distance between two probability distributions $\mathbb{P}$ and $\mathbb{Q}$. It is defined as $$\mathsf{KL}(\mathbb{P}\;||\; \mathbb{Q}) \coloneqq \mathbb{E}_{\mathbb{P}}[\log\frac{d\mathbb{P}}{d\mathbb{Q}}]$$ where $\frac{d\mathbb{P}}{d\mathbb{Q}}$ is the Radon-Nikodym derivative of $\mathbb{P}$ with respect to $\mathbb{Q}$.  The Shannon entropy of a random variable is the amount of information contained in $X$ and is defined as $H(X)\coloneqq \mathbb{E}_{\mathbb{P}_X}[-\log d\mathbb{P}_X]$, where $\mathbb{P}_X$ is the distribution of $X$. Mutual information between two random variables $X$, $Y$ is defined as 
	$$I(X;Y) \coloneqq H(X) - H(X\;|\;Y)$$
	where $H(X)$ is the Shannon entropy of $X$, $H(X \;|\;Y)$ is the Shannon entropy of $X$ conditional on $Y$.
	Let the joint probability distribution of $X$ and $Y$ be $\mathbb{P}_{XY}$ and the product of the marginal distributions be $\mathbb{P}_X \otimes \mathbb{P}_Y$. The mutual information between two random variables can also be expressed in terms of the KL divergence as follows. 
	$I(X;Y) = \mathsf{KL}(\mathbb{P}_{XY} \;|| \; \mathbb{P}_{X}\otimes \mathbb{P}_{Y})$, where $\mathsf{KL}$ is the KL divergence between the two input distributions. We  describe the  Donsker-Varadhan representation for KL divergence next.
	
	\subsection{Donsker-Varadhan Representation}
	The Donsker Varadhan (DV) representation \cite{donsker1983asymptotic} for KL divergence between two distributions $\mathbb{P}$ and $\mathbb{Q}$ is given as follows. The sample space for the distributions $\mathbb{P}$ and $\mathbb{Q}$  is the same set $\Omega$. For simplicity, we assume that $\Omega$ is a compact subset of  $\mathbb{R}^{s}$. 
	Suppose $T$ is a mapping from the sample space $\Omega$ to $\mathbb{R}$, i.e., $T:\Omega \rightarrow \mathbb{R}$. 
	
	\begin{equation}
	\mathsf{KL}(\mathbb{P}\;||\;\mathbb{Q}) = \sup_{T \in \mathcal{M}} \Big[\mathbb{E}_{\mathbb{P}}\big[T \big]- \log\big( \mathbb{E}_{\mathbb{Q}}\big[e^{T}\big]\big) \Big]
	\label{kl_dv}
	\end{equation}
	where $\mathcal{M}$ is the space of mappings where both the expectations $\mathbb{E}_{\mathbb{P}}\big[T \big]$ and $\log\big( \mathbb{E}_{\mathbb{Q}}\big[e^{T}\big]\big)$ are finite. Recall that if $\mathbb{P}= \mathbb{P}_{XY}$ and $\mathbb{Q} = \mathbb{P}_{X} \otimes \mathbb{P}_{Y}$, then we obtain the mutual information $I(X;Y)$. Since our work is closely related to MINE \cite{belghazi2018mine} we explain the approach briefly in the next section.

	
	\subsection{MINE}
	We are given a  set of parameters  $\Theta$ that define the family of neural networks. Each member $\theta$ of the family  characterizes a function $T_{\theta}$ and the set of all the functions is defined as $\mathcal{F} = \{T_{\theta}; \theta \in \Theta\}$.  The neural measure of KL divergence is defined as 
	
	\begin{equation}
	\mathsf{KL}_{\Theta}(\mathbb{P}\;||\;\mathbb{Q})  =  \sup_{\theta \in \Theta} \Big[ \mathbb{E}_{\mathbb{P}}\big[T_{\theta} \big]- \log\big( \mathbb{E}_{\mathbb{Q} }\big[e^{T_{\theta}}\big]\big) \Big]
	\label{neural_kl}
	\end{equation}
	From \eqref{kl_dv} and \eqref{neural_kl}, we can see that $$\mathsf{KL}(\mathbb{P}\;||\;\mathbb{Q}) \geq \mathsf{KL}_{\Theta}(\mathbb{P}\;||\;\mathbb{Q})$$
	Define $\hat{\mathbb{P}}^{(n)}$ and $\hat{\mathbb{Q}}^{(m)}$ as the empirical distribution of $\mathbb{P}$ and $\mathbb{Q}$ respectively with $n$ and $m$ i.i.d. samples given as $\boldsymbol{X}=\{x_i\}_{i=1}^{n}$ and $\boldsymbol{Y}=\{y_j\}_{j=1}^m$ respectively. Let $\boldsymbol{Z} = \boldsymbol{X} \cup \boldsymbol{Y}$. We write $\boldsymbol{Z} = \{z_k, \forall k \in \{1,..,n+m\}\}$, where $z_k = x_k, \; \forall k \in \{1,...,n\}$ and $z_{n+k}= y_k, \; \forall k \in \{1,...,m\}$..
	The MINE estimator for KL divergence is given as 
	
	\begin{equation}
	\hat{\mathsf{KL}}_{\Theta}(\hat{\mathbb{P}}^{n}\;||\;\hat{\mathbb{Q}}^{m}) = \sup_{\theta \in \Theta}\Big[\mathbb{E}_{\hat{\mathbb{P}}^{n}}\big[T_{\theta} \big]- \log\big( \mathbb{E}_{\hat{\mathbb{Q}}^{m}}\big[e^{T_{\theta}}\big]\big)\Big]
	\label{MINE_estimator}
	\end{equation} 
	\subsubsection{Limitations of MINE}
	In \cite{belghazi2018mine}, it was shown that $ \hat{\mathsf{KL}}_{\Theta}(\hat{\mathbb{P}}^{n}\;||\;\hat{\mathbb{Q}}^{m})$ is a consistent estimator of the KL divergence.  The algorithm in \cite{belghazi2018mine} tries to maximize the  loss function  $\mathbb{E}_{\hat{\mathbb{P}}^{n}}\big[T_{\theta} \big]- \log\big( \mathbb{E}_{\hat{\mathbb{Q}}^{m}}\big[e^{T_{\theta}}\big]\big) 
	$ to get as close as possible to \eqref{MINE_estimator}.  Stochastic gradient descent (SGD) is used to search for the optimal neural network parameters $\theta$ in $\Theta$. For the estimator in \eqref{MINE_estimator} to be consistent the family of neural networks has to consist of at least one hidden layer \cite{belghazi2018mine} \cite{hornik1991approximation}. As a result, the loss function that the algorithm tries to optimize is non-convex in the parameters of the neural network.
	Since the loss is non-convex it is not guaranteed to converge to the MINE estimator defined in equation (\ref{MINE_estimator}).  Also, since the loss function is non-convex the optimization can lead to poor local minima, which are worse than the other minima or have poor generalization properties.

	%

	%

	\subsection{KKLE: Kernel Based KL Divergence Estimation}
	In this section, we build an approach that overcomes the limitations that were highlighted in the previous section.
	Consider  a RKHS $\mathcal{H}$  over $\mathbb{R}$ with a kernel $k:\mathbb{R}\times \mathbb{R} \rightarrow \mathbb{R}$.  We assume that the kernel is a continuously differentiable function. The norm of a function $T$ in $\mathcal{H}$ is given as $\|T\|_{\mathcal{H}}^2 = \braket{T,T}_{\mathcal{H}} $, where $\braket{}_{\mathcal{H}}$ is the inner product defined in the Hilbert Space. In \cite{belghazi2018mine}, it was assumed that the function $T_{\theta}$ is bounded. We also limit our search over the space of  bounded functions, i.e., we assume that the $\|T\|_{\mathcal{H}} \leq M$. This is a reasonable assumption to make because \eqref{kl_dv} assumes the two expectation terms are finite, which is only possible if $T$ is bounded almost everywhere. 
	We define the kernel measure of KL divergence as follows 
	\begin{equation}
	\mathsf{KL}_{\mathcal{H}} (\mathbb{P}\;||\mathbb{Q}) =  \sup_{T \in \mathcal{H},\|T\|_{\mathcal{H}\leq M}} \mathbb{E}_{\mathbb{P}}\big[T \big]- \log\big( \mathbb{E}_{\mathbb{Q}}\big[e^{T}\big]\big)
	\label{kernel_information_measure}
	\end{equation}
	From \eqref{kernel_information_measure} and \eqref{kl_dv}, we can also deduce that 
	$$\mathsf{KL} (\mathbb{P}\;||\;\mathbb{Q}) \geq \mathsf{KL}_{\mathcal{H}} (\mathbb{P}\;||\;\mathbb{Q})$$
	
	We define the empirical estimator of the kernel measure below. 
	
	\begin{equation}
	\hat{\mathsf{KL}}_{\mathcal{H}} (\hat{\mathbb{P}}^n \; || \; \hat{\mathbb{Q}}^{m} ) =  \sup_{T \in \mathcal{H}, \|T\|_{\mathcal{H}\leq M}} \Big[\mathbb{E}_{\hat{\mathbb{P}}^{n}}\big[T \big]- \log\big( \mathbb{E}_{\hat{\mathbb{Q}}^{m} }\big[e^{T}\big]\big)\Big]
	\label{kernel_estimator}
	\end{equation}
	
	We define a matrix $\boldsymbol{K}$, which we call the kernel matrix, such that for every  $z_i \in \boldsymbol{Z}$, $z_j \in \boldsymbol{Z}$,  $\boldsymbol{K}[z_i,z_j] = k(z_i, z_j)$. For the rest of the discussion, we assume that the maximum exists  and hence, the supremum and maximum are interchangeable.  Let $$g(\boldsymbol{\alpha}) =\log(\frac{1}{m} \sum_{y_j \in \boldsymbol{Y}}e^{\boldsymbol{\alpha}^{t}\boldsymbol{K}[y_j,:]}) -\frac{1}{n} \sum_{x_i \in \boldsymbol{X}} \boldsymbol{\alpha}^{t}\boldsymbol{K}[x_i,:]$$
	In the next proposition, we show that we can compute $ \hat{\mathsf{KL}}_{\mathcal{H}} (\hat{\mathbb{P}}^n \; || \; \hat{\mathbb{Q}}^{m} )$ by minimizing  $g(\boldsymbol{\alpha})$.
	
	\begin{proposition}
		For any $\epsilon>0$, $\exists\; t>0$ such that   	the optimal $T$ that solves \eqref{kernel_estimator} is  $T^{*}(z) = \sum_{i=1}^{n+m}\alpha_i^{*} k(z_i,z) $, where $\boldsymbol{\alpha}^{*}$ is \begin{equation}
		\boldsymbol{\alpha}^{*}=	\arg\min_{\boldsymbol{\alpha}, \boldsymbol{\alpha}^{t}\boldsymbol{K}\boldsymbol{\alpha} \leq M^2} g(\boldsymbol{\alpha}) + \frac{1}{t} \boldsymbol{\alpha}^t \boldsymbol{K}\boldsymbol{\alpha}
		\label{optimization_kernel_new}
		\end{equation}
		and
		$$|	\hat{\mathsf{KL}}_{\mathcal{H}} (\hat{\mathbb{P}}^n \; || \; \hat{\mathbb{Q}}^{m} )  + g(\boldsymbol{\alpha}^{*})| \leq \epsilon $$
	\end{proposition}


	\begin{proof}
		We rewrite the objective in \eqref{kernel_estimator} as a penalized objective as follows. \begin{equation}
		\log\big( \mathbb{E}_{\hat{\mathbb{Q}}^{m} }\big[e^{T}\big]\big) -\mathbb{E}_{\hat{\mathbb{P}}^{n}}\big[T \big]  +\frac{1}{t} \|T\|_{\mathcal{H}}^2
		\label{penalized_kernel}
		\end{equation}
		Suppose that $t$ is large  enough, i.e., $t\geq M^2/\epsilon$. Therefore, the penalty term is bounded by a small value $\epsilon$. In such a case,  the negative of the penalized objective in \eqref{penalized_kernel} is very close to the objective in \eqref{kernel_estimator}. Therefore, solving the problem below should give an $\epsilon$ approximate solution to \eqref{kernel_estimator}.
		\begin{equation}
		\min_{T, \|T\|_{\mathcal{H}} \leq M} \log\big( \mathbb{E}_{\hat{\mathbb{Q}}^{m} }\big[e^{T}\big]\big) -\mathbb{E}_{\hat{\mathbb{P}}^{n}}\big[T \big]  +\frac{1}{t} \|T\|_{\mathcal{H}}^2
		\label{optimal_penalized_kernel}
		\end{equation}

		We  use Representer Theorem (See \cite{scholkopf2001learning}) to infer that the optimal $T$ for \eqref{optimal_penalized_kernel} that achieves the minimum above can be written as a linear combination 
		\begin{equation}
		T^{*}(.) = \sum_{i=1}^{n+m}\alpha_{i} k(z_i, .)
		\label{kernel_linear_relation}
		\end{equation}
		where $z_i = x_i, \; \forall i \in \{1,...,n\}$ and $z_{n+j}= y_j, \; \forall j \in \{1,...,m\}$.
		We substitute the above expressions from \eqref{kernel_linear_relation} in \eqref{optimal_penalized_kernel} to obtain the following equivalent optimization problem.
		
		\begin{equation}
		\begin{split}
		\min_{\boldsymbol{\alpha},  \boldsymbol{\alpha}^{t}K\boldsymbol{\alpha} \leq M^2}& \log(\frac{1}{m} \sum_{y_j \in \boldsymbol{Y}}e^{\boldsymbol{\alpha}^{t}\boldsymbol{K}[y_j,:]}) - \frac{1}{n}\sum_{x_i \in \boldsymbol{X}} \boldsymbol{\alpha}^{t}\boldsymbol{K}[x_i,:] + \frac{1}{t} \boldsymbol{\alpha}^t \boldsymbol{K}\boldsymbol{\alpha}
		\end{split}
		\label{penalized_optimization_kernel}
		\end{equation}
		Hence, \eqref{penalized_optimization_kernel} is equivalent to 	\eqref{optimal_penalized_kernel}, which gives the $\epsilon$ approximate optimal solution to \eqref{kernel_estimator}. This completes the proof.
	\hfill	$\blacksquare$
	\end{proof} 
	In Proposition 1, we showed that , i.e., $ \hat{\mathsf{KL}}_{\mathcal{H}} (\hat{\mathbb{P}}^n \; || \; \hat{\mathbb{Q}}^{m} ) \approx -g(\boldsymbol{\alpha}^{*})$. Next we discuss how to solve for $\hat{\mathsf{KL}}_{\mathcal{H}} (\hat{\mathbb{P}}^n \; || \; \hat{\mathbb{Q}}^{m} )$ efficiently.
	We  solve \eqref{optimization_kernel_new}  using SGD. See Algorithm 1 for a detailed description. 
	\begin{algorithm}
		\caption{KKLE algorithm to estimate KL divergence}
		\begin{algorithmic}[1]
			\renewcommand{\algorithmicrequire}{\textbf{Input:}}
			\renewcommand{\algorithmicensure}{\textbf{Output:}}
			
			\REQUIRE $\boldsymbol{X}=\{x_i\}_{i=1}^{n}\sim\mathbb{P}$ and $\boldsymbol{Y}=\{y_j\}_{j=1}^{m}\sim \mathbb{Q}$, $\gamma$ (distance from minimum), $\mathsf{max}_{\mathsf{iter}}$ (maximum number of steps), $\eta$ (step size), $k$ (minibatch size)
			\ENSURE  KL divergence estimate
			\\ \textit{Initialization:} Initialize $\boldsymbol{\alpha}$ randomly, $\mathsf{n}_{\mathsf{iter}}=0$, $\mathsf{Convergence}=\mathsf{False}$
			\STATE \textbf{While ($\mathsf{n}_{\mathsf{iter}} \leq \mathsf{max}_{\mathsf{iter}}$ and $\mathsf{Convergence}==\mathsf{False}$)}
			\STATE \textbf{Minibatch sampling:} Sample $k$ samples from $\boldsymbol{X}$ and  $k$ samples from $\boldsymbol{Y}$ 
			\STATE$$\hat{\mathsf{KL}}(\boldsymbol{\alpha})_p = -\log(\frac{1}{m} \sum_{y_j \in \boldsymbol{Y}}e^{\boldsymbol{\alpha}^{t}\boldsymbol{K}[y_j,:]}) + \frac{1}{n}\sum_{x_i \in \boldsymbol{X}} \boldsymbol{\alpha}^{t}\boldsymbol{K}[x_i,:] + \frac{1}{t} \boldsymbol{\alpha}^t \boldsymbol{K}\boldsymbol{\alpha}$$
			\STATE $\alpha = \alpha + \eta\nabla \hat{\mathsf{KL}}(\boldsymbol{\alpha})_p  $
			\STATE $$\hat{\mathsf{KL}}(\boldsymbol{\alpha})_c = -\log(\frac{1}{m} \sum_{y_j \in \boldsymbol{Y}}e^{\boldsymbol{\alpha}^{t}\boldsymbol{K}[y_j,:]}) + \frac{1}{n} \sum_{x_i \in \boldsymbol{X}} \boldsymbol{\alpha}^{t}\boldsymbol{K}[x_i,:] + \frac{1}{t} \boldsymbol{\alpha}^t \boldsymbol{K}\boldsymbol{\alpha}$$
			\STATE   \textbf{If} $\;$ $|\hat{\mathsf{KL}}(\boldsymbol{\alpha})_c-\hat{\mathsf{KL}}(\boldsymbol{\alpha})_p|\leq \gamma$ 
			\STATE $\;\;$ $\mathsf{Convergence} = \mathsf{True}$
			\STATE $\mathsf{n}_{\mathsf{iter}}= \mathsf{n}_{\mathsf{iter}}+1$
			\RETURN $\hat{\mathsf{KL}}(\boldsymbol{\alpha})_{c}$
		\end{algorithmic} 
	\end{algorithm}

	\vspace{1em}
	\begin{proposition} \begin{itemize}\item The optimization problem in \eqref{optimization_kernel_new} is a convex optimization problem. 
			\item Algorithm 1 converges to the optimal solution of \eqref{optimization_kernel_new}.  
		\end{itemize}
	\end{proposition}
	
	\begin{proof} The first term in the objective in \eqref{optimization_kernel_new} is
		log of sum of exponentials, which is  a convex function (See \cite{boyd2004convex}).  The second term in \eqref{optimization_kernel_new} is linear. Therefore, the objective in \eqref{optimization_kernel_new} is a convex function.
		The matrix $\boldsymbol{K}$ is positive definite (See \cite{scholkopf2001learning}). Hence, the function $\boldsymbol{\alpha}^{t} \boldsymbol{K} \boldsymbol{\alpha}$ is convex. Therefore, the set of $\boldsymbol{\alpha}$ to be searched, i.e., $\boldsymbol{\alpha}^{t} \boldsymbol{K}  \boldsymbol{\alpha} \leq M^2$ is a convex set. This establishes that \eqref{optimization_kernel_new} is a convex optimization problem. 
		
		If the objective function \eqref{optimization_kernel_new} is Lipschitz continuous and convex and bounded, then the Algorithm 1  based procedure would converge to the minimum (See Chapter 14 in \cite{shalev2014understanding}). 
		We want to show that 
		$g(\boldsymbol{\alpha}) = \log(\frac{1}{m} \sum_{y_j \in \boldsymbol{Y}}e^{\boldsymbol{\alpha}^{t}\boldsymbol{K}[y_j,:]}) - \frac{1}{n}\sum_{x_i \in \boldsymbol{X}} \boldsymbol{\alpha}^{t}\boldsymbol{K}[x_i,:]$ is Lipschitz continuous in $\boldsymbol{\alpha}$. It is sufficient to show that the gradient of the function $g$ w.r.t $\alpha$ is bounded.
		Define a function  $$h(t)  = g(x + t(y-x))$$ and $h^{'}(t) = dh(t)/dt$. Observe that $h(0) = g(x)$ and $h(1) = g(y)$. Using chain rule we can write $h^{'}(t) = \nabla_{z} g(z)^{t}|_{z = x+ t(y-x)} (y-x)$
		\begin{equation}
		\begin{split}
		&g(y) - g(x)= \int_{0}^{1} h^{'}(t) dt \\&=  \int_{0}^{1}  \nabla_{z} g(z)^{t}|_{z  = x+ t(y-x)} (y-x) dt  \leq \| \nabla_{z} g(z)\|\|y-x\|
		\end{split}
		\label{lipschitz_eqn}
		\end{equation}

		We write the partial derivative of $g$ w.r.t. each component of $\boldsymbol{\alpha}$ as follows 
		$\frac{\partial g(\alpha) }{\alpha_j} = \frac{\sum_{i=1}^{n+m}e^{\alpha_{j} \boldsymbol{K}[z_i, z_j]}\boldsymbol{K}[z_i, z_j]}{ \sum_{i=1}^{n+m}e^{\alpha^{t}\boldsymbol{K}[z_i,:]}}$. We want to derive a loose  upper bound on $\|\nabla g\|_1$.  To do that we first make the following observation about the matrix $\boldsymbol{K}$.
		We assumed that the samples $x_i$ and $y_j$ that are drawn from the distributions $\mathbb{P}$ and $\mathbb{Q}$ come from a set $\Omega$, which is a compact subset of $\mathbb{R}^{s}$. Since the kernel $k$ is a continuously differentiable function and $\Omega$ is a compact subset we can infer that all the elements in $\boldsymbol{K}$ are bounded.  For simplicity, we assume that $K $ is bounded above by 1 and bounded below by zero. 
		Since all the terms in $\frac{\partial g(\alpha) }{\alpha_j} $ are positive we can say the following

		\begin{equation}
		\begin{split}
		& \|\nabla g\|_1=\frac{ \sum_{j=1}^{n+m}\sum_{i=1}^{n+m}e^{\alpha_{j}\boldsymbol{K}[z_i,z_j]}\boldsymbol{K}[z_i,z_j]}{ \sum_{i=1}^{n}e^{\alpha^{t}\boldsymbol{K}[z_i,:]}} \leq \frac{ \sum_{j=1}^{n}\sum_{i=1}^{n}e^{\alpha_{j}\boldsymbol{K}[z_i,z_j]}\boldsymbol{K}[z_i,z_j]}{ n }\leq \frac{\sum_{j=1}^{n}\sum_{i=1}^{n}e^{\alpha_{j}}}{n}\\ 
		& \leq \max_{\boldsymbol{\alpha}, \boldsymbol{\alpha}^t\boldsymbol{K}\boldsymbol{\alpha} \leq M} \sum_{j=1}^{n} e^{\alpha_j}
		\end{split}
		\label{max_lipsch}
		\end{equation}
		
		Since $\sum_{j=1}^{n} e^{\alpha_j}$ is bounded above in the search space. Therefore, the maximum in \eqref{max_lipsch} has to be finite. Since $\|\nabla g\|_2 \leq \|\nabla g\|_1 $. Hence $\|\nabla g\|_2 $ is bounded above and from \eqref{lipschitz_eqn} we can see that the function $g$ is Lipschitz continuous in $\boldsymbol{\alpha}$. Lastly, it is easy to see that $g$ itself is bounded because $\boldsymbol{K}$ is bounded and $\boldsymbol{\alpha}$ also takes value in a compact set. 
We also need to show that the second term in \eqref{optimization_kernel_new} is also Lipschitz continuous. The gradient of the second term is $2\boldsymbol{K}\boldsymbol{\alpha}$. Let us try to bound the norm of the gradient. 
Before that since we know that $\boldsymbol{K}$ is positive definite and symmetric, we can write the eigendecomposition of $\boldsymbol{K}$ as $\boldsymbol{K} = U \Lambda U^t$, where $U$ is an orthonomal matrix comprised of the eigenvectors of $\boldsymbol{K}$, $\Lambda = \text{diag}[\lambda_1,..., \lambda_{m+n}]$ is the diagonal matrix of the set of eigenvalues $\{\lambda_{i}\}_{i=1}^{m+n}$. 
\begin{equation}
\|\boldsymbol{K}\boldsymbol{\alpha}\|^2 = \boldsymbol{\alpha}^t \boldsymbol{K}^t\boldsymbol{K} \boldsymbol{\alpha} = \boldsymbol{\alpha}^tU^t \Lambda^2 U \boldsymbol{\alpha} = v^{t}\Lambda^2 v \leq \sum_{i}\lambda_{i}^2 \|v\|^2 =  \sum_{i}\lambda_{i}^2 \|\boldsymbol{\alpha}\|^2
\end{equation}

In the last simplification on RHS in the above we use the following.  $v = U \boldsymbol{\alpha}$ and  $\|U\boldsymbol{\alpha}\| = \|\boldsymbol{\alpha}\|$ ($U$ is an orthonormal matrix).   $\boldsymbol{\alpha}^{t} \boldsymbol{K}\boldsymbol{\alpha}$ is bounded $\implies$ $\|\alpha\|$ is also bounded. Hence, $
\|\boldsymbol{K}\boldsymbol{\alpha}\|^2$ is also bounded. 
We have now shown that the objective in \eqref{optimization_kernel_new} is Lipschitz continuous.  From Corollary 14.2 in \cite{shalev2014understanding}, we know that the procedure in Algorithm 1 \footnote{For the proof we are assuming that we use the entire data in one minibatch and follow gradient descent.} converges to the minimum of the  problem \eqref{optimization_kernel_new}. \hfill$\blacksquare$

	\end{proof}


	\subsection{Analyzing the Consistency of KKLE}

	\begin{definition} Strong Consistency:	For all $\eta>0$, if there exists a kernel $k$ and an $N$ such that $\forall n\geq N,m\geq N $such that $	|\hat{\mathsf{KL}}_{\mathcal{H}}(\hat{\mathbb{P}}^{n}\;||\; \hat{\mathbb{Q}}^{m}) -\mathsf{KL}(\mathbb{P}\;||\;\mathbb{Q})| \leq \eta$ then  $\hat{\mathsf{KL}}_{\mathcal{H}}(\hat{\mathbb{P}}^{n}\;||\; \hat{\mathbb{Q}}^{m})$ is a strongly consistent estimator of $\mathsf{KL}(\mathbb{P}\;||\;\mathbb{Q})$
	\end{definition}
	\begin{proposition}
		$\hat{\mathsf{KL}}_{\mathcal{H}}(\hat{\mathbb{P}}^{n}\;||\; \hat{\mathbb{Q}}^{m})$ is a strongly consistent estimator of $\mathsf{KL}(\mathbb{P}\;||\;\mathbb{Q})$
	\end{proposition}
	
	\begin{proof}
		The proof of this Proposition follows the same steps as the Proof in \cite{belghazi2018mine}. 
		Since we are in a setting where the consistency depends on the expressiveness of the Hilbert Space, which is different from the setting in \cite{belghazi2018mine}, we have to redo the proof for this case.   We divide the proof into two parts. 
		
		For simplicity, we assume that the Hilbert space $\mathcal{H}$ has a finite dimensional basis $\boldsymbol{\Phi}$. Hence, every function in $\mathcal{H}$ can be written as $T(z) = \beta^{t}\boldsymbol{\Phi(z)}$.   We substitute this form of function in \eqref{kernel_estimator} to obtain
		\begin{equation}
		\begin{split}
		&\hat{\mathsf{KL}}_{\mathcal{H}}(\hat{\mathbb{P}}^n\; || \; \hat{\mathbb{Q}}^m)=-\min_{\boldsymbol{\beta}, \|\beta\| \leq M}\Big[ \log(\frac{1}{m} \sum_{y_j \in \boldsymbol{Y}}e^{\boldsymbol{\beta}^{t}\boldsymbol{\Phi}(y_j)}) - \frac{1}{n}\sum_{x_i \in \boldsymbol{X}} \boldsymbol{\beta}^{t}\boldsymbol{\Phi}(x_i)\Big]
		\end{split}
		\label{optimization_kernel}
		\end{equation}
	\end{proof}
	
	Note that the assumption will not limit us from extending the proof to infinite basis (We can approximate an infinite radial basis function kernel with a finite radial basis \cite{rahimi2008random}). Next we show that the  estimator from   \eqref{optimization_kernel} is a consistent estimator of \eqref{kernel_information_measure}.	
	%
	




	We use the triangle inequality to arrive at the following.
	\begin{equation}
	\begin{split}
	&|\hat{\mathsf{KL}}_{\mathcal{H}}(\hat{\mathbb{P}}^n\; || \; \hat{\mathbb{Q}}^m) -\mathsf{KL}_{\mathcal{H}} (\mathbb{P}\;||\;\mathbb{Q}) | \leq \max_{\beta, \|\beta\|\leq M}\Big(|\frac{1}{n}\big[\sum_{x_i \in \boldsymbol{X}} \boldsymbol{\beta}^{t}\boldsymbol{\Phi}(x_i)\big]-\mathbb{E} \big[ \boldsymbol{\beta}^{t}\boldsymbol{\Phi}(x_i)\big]| \Big) \\ 
	& + \max_{\beta, \|\beta\|\leq M}|\log(\frac{1}{m} \sum_{y_j \in \boldsymbol{Y}}e^{\boldsymbol{\beta}^{t}\boldsymbol{\Phi}(y_j)}) -\log(\mathbb{E} \big[e^{\boldsymbol{\beta}^{t}\boldsymbol{\Phi}(y_j)}\big]) | 
	\end{split}
	\end{equation}
	$\boldsymbol{\Phi}$  is a continuous function and since the outcomes are drawn from $\Omega$, a compact subset in $\mathbb{R}$, $\boldsymbol{\Phi}$ is bounded.
	$\boldsymbol{\beta}^{t} \boldsymbol{\Phi}$ is  bounded over the set $\|\boldsymbol{\beta}\| \leq M$. The space of parameters $\boldsymbol{\beta}$ is compact because the norm of $\|\boldsymbol{\beta}\|$ is bounded.  These  observations allow us to use  \cite{van2000empirical} to show the following for a sufficiently large $N$ and $n\geq N$ 
	$$
	\max_{\beta, \|\beta\| \leq M}\Big(|\frac{1}{n}\big[\sum_{x_i \in \boldsymbol{X}} \boldsymbol{\beta}^{t}\boldsymbol{\Phi}(x_i)\big]-\mathbb{E} \big[ \boldsymbol{\beta}^{t}\boldsymbol{\Phi}(x_i)\big]| \Big) \leq \eta/2$$
	

	Similarly $\log(\mathbb{E} \big[e^{\boldsymbol{\beta}^{t}\boldsymbol{\Phi}_{i}}\big])$ is also bounded in  $\|\boldsymbol{\beta}\| \leq M$.
	
	Similarly, for a sufficiently large $N$ and $m\geq N$ we have $$\max_{\beta, \|\beta\|\leq M}|\log(\frac{1}{m} \sum_{y_j \in \boldsymbol{Y}}e^{\boldsymbol{\beta}^{t}\boldsymbol{\Phi}(y_j)}) -\log(\mathbb{E} \big[e^{\boldsymbol{\beta}^{t}\boldsymbol{\Phi}(y_j)}\big]) | \leq \eta/2 $$

	The next question we are interested in is if there exists a finite basis that is good enough.  We use radial basis functions (Gaussian radial basis in particular) with finite number of centers  \cite{wu2012using}.  Suppose we use a weighted sum of the radial basis functions to learn the mutual information.    	In \cite{buhmann2003radial} \cite{park1991universal} \cite{wu2012using}, it is shown that finite radial basis functions can approximate arbitrary functions.        
	We  assume that the function that achieves optimal for \eqref{kl_dv} is smooth (This assumption is also made in  \cite{hornik1991approximation} and \cite{belghazi2018mine}).

	
	%
	
	
	Let $T^{*} = \log\frac{d\mathbb{P}}{d\mathbb{Q}}$. By construction $T^{*}$ satisfies
	
	$\mathbb{E}_{\mathbb{P}} [T^{*}] = \mathsf{KL}(\mathbb{P}\;||\;\mathbb{Q})$ and 
	$\mathbb{E}_{\mathbb{Q}}[e^{T^{*}}] = 1$. Suppose we allow for $\eta$ tolerance on the error on the function we want to approximate.  For a fixed $\eta$, we can derive a finite basis which can approximate any smooth function as shown in \cite{park1991universal}.  Suppose a finite radial basis function spans the RKHS and let $T$ be the function that achieves the maximizer for \eqref{kernel_estimator}.
	For a function $T$ we can write the gap between the true KL divergence and KL divergence achieved by restricting the search to the Hilbert Space as
	$$\mathsf{KL}(\mathbb{P}\;||\;\mathbb{Q}) - \mathsf{KL}_{\mathcal{H}}(\mathbb{P}\;||\;\mathbb{Q}) = \mathbb{E}_{\mathbb{P}} [T^{*}- T] + \mathbb{E}_{\mathbb{Q}} [e^{T^{*}}- e^T]  $$

	We can select a large enough radial basis (Theorem 1 in \cite{park1991universal}) such that $$ \mathbb{E}_{\mathbb{P}} [T^{*}- T] \leq \eta/2$$
	$$ \mathbb{E}_{\mathbb{Q}} [e^{T^{*}}- e^T] \leq \eta/2$$
	Both the above conditions hold simultaneously because $e^x$ is Lipschitz continuous and $T$ is bounded $\|T\|_{\mathcal{H}}\leq M$.\hfill$\blacksquare$
	
	We established that the proposed estimator is strongly consistent. In the next section, we analyze the complexity and convergence properties of KKLE.

	\subsection{Convergence and Complexity}

	The approach in Algorithm 1 optimizes the objective in \eqref{optimization_kernel_new}.  The number of steps before which the algorithm is guaranteed to converge is computed using \cite{shalev2014understanding}. The number steps grow as $\mathcal{O}(\frac{\rho^2}{\epsilon^2})$, where $\rho$ is the Lipschitz constant for the loss function and $\gamma$ is the tolerance in maximum distance from the minimum value of the loss (also defined in Algorithm 1).

	The dimension of $\boldsymbol{\alpha}$ vector is $n+m$ and the dimension of the kernel matrix $\boldsymbol{K}$  is $m+n \times m+n$.  Computing and storing this matrix can be a problem if the data is too large. 
	The time complexity of the algorithm is given as $\mathcal{O}(\mathsf{max_{iter}} (m+n)^2 )$, where $\mathsf{max_{iter}}$ is the maximum number of steps in the Algorithm 1 and $(m+n)^2$ is the computational cost per step.

	If the size of the data is large, then solving the above problem can be slow. We  use \cite{rahimi2008random} to improve the computational speed. In \cite{rahimi2008random}, the authors derive an approximation in terms of a lower $d$ dimensional mapping $\phi$ to approximately reproduce the kernel $k$. The complexity with this approximation drops to   $\mathcal{O}(\mathsf{max_{iter}} (m+n)d  )$.  In the experiments section, we use this trick to improve the complexity. 
	
	Before going to experiments, we conclude this section with an illustrative comparison of KKLE with MINE.
	In Figure \ref{fig}, we compare the two estimators (KKLE and MINE) for the case when RKHS is finite dimensional. For MINE all the layers of the neural network are trained to optimize the objective \eqref{MINE_estimator}. For KKLE, the first layer projects the data into a higher dimensional basis of RKHS. The second and the final layer is trained to optimize \eqref{kernel_estimator}.

	%
	


	%
	%
	%
	%
	
		\begin{figure}
		\centering
		\includegraphics[width=5in]{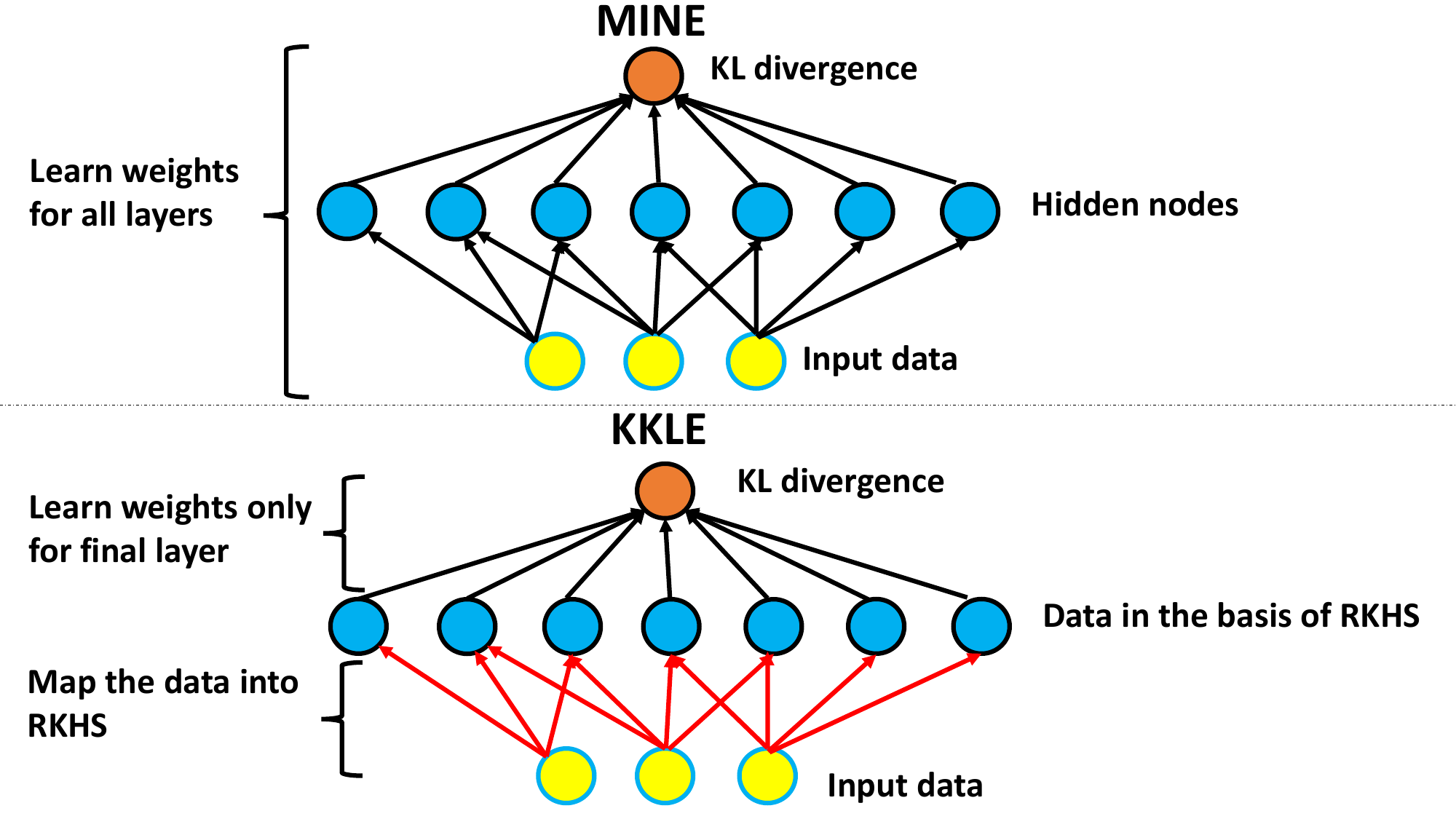}
		\caption{KKLE vs MINE: Comparison of the estimation algorithms}
		\label{fig}
	\end{figure}
	\section{Experiments}
	\subsection{Comparisons}
	\subsubsection{Setup} We use the same setting as in \cite{poole2018variational} \cite{belghazi2018mine}. We compare MINE estimator with KKLE estimator on the task of estimating mutual information,  which as described earlier can also be represented in terms of the KL divergence. There are two random vectors $\boldsymbol{X} \in \mathbb{R}^{D}$ and $\boldsymbol{Y} \in \mathbb{R}^{D}$, where $X_k$ and $Y_k$ are the $k^{th}$ components of $\boldsymbol{X}$ and $\boldsymbol{Y}$ respectively. $(X_k,Y_k)$ is drawn from a 2-dimensional Gaussian distribution with $0$ mean and $\rho$ correlation. The true mutual information in this case can be analytically computed and is given as $-\frac{D}{2} \log(1-\rho^2)$. We are given a dataset with $N$ i.i.d. samples from the distribution of $(\boldsymbol{X}, \boldsymbol{Y})$. 
	In the next section, we compare the performance of the proposed KKLE estimator with MINE estimator in terms of the following metrics: Bias of the estimator, root mean squared error in the estimation (RMSE), variance in the estimator values. 
	All the simulations were done on a 2.2GHz Intel Core i7 processor, with 16 GB memory using Tensorflow in Python. We use \cite{rahimi2008random} to map the features and reduce the computational costs. The comparisons are done for two scenarios, when the dataset is very large, and when the dataset is small. 

	
	\subsubsection{Comparisons for large data} 
	
	In this section, our goal is to compare the two estimators for a  sufficiently large dataset ($N=10^5$) to show both the estimators are consistent.  We sample $N=10^5$  $(\boldsymbol{X}, \boldsymbol{Y})$ from the distribution described above for $D=1$ and $D=5$.
	We  compare the bias, RMSE, and variance of the proposed KKLE estimator with the MINE estimator.  The minibatch size for the gradient descent is 5000. In each step a minibatch is sampled and a gradient step is taken. The total number of gradient steps is 1000. In Table 1, we provide the comparisons for $D=1$ and $D=5$. The results in the Table 1 are averaged over 100 trials. We observe that the performance of both the estimators are similar. Note that both the estimators degrade in the setting when dimensionality of the data becomes large.

	\begin{table*}
	\centering	\caption{KKLE vs MINE estimator for large data}
	\begin{tabular}
		{cccccccc}
		\toprule
		Estimator & D &      Bias &      RMSE &  Variance &  Correlation &  Mutual Information \\
		\midrule
		MINE &1  	& -0.009442 &  0.011378 &  0.000040 &        0.2 &                 0.020411 \\
		MINE &1 	 &0.025266 &  0.030608 &  0.000299&     0.5 &                 0.143841 \\
		MINE &1	 & -0.060696 &  0.075414 &  0.002003 &      0.9 &                 0.830366 \\
		
		\midrule
		KKLE &1	 & -0.009221 &  0.010990 &  0.000036 &        0.2 &                 0.020411 \\
		KKLE  &1 	 & -0.025688 &  0.030982 &  0.000300  &      0.5 &                 0.143841 \\
		KKLE &1 	 & -0.065784 &  0.079743 &  0.002031 &        0.9 &                 0.830366 \\
		\bottomrule
		
		
		MINE & 	5 & -0.020874 &  0.024841&    0.000181    & 0.2 &                 0.102055 \\
		MINE & 	5 & -0.072369&  0.09106 &      0.003055   &0.5 &                 0.719205 \\
		MINE &	5 & -0.415350 &  0.758026 &      0.402088   &  0.9 &                 4.151828 \\
		\midrule
		KKLE		&5 & -0.006116  & 0.038716  &    0.00146  &  0.2  &              0.102055   \\
		KKLE		&5 & -0.046382&  0.116801&    0.011491    &  0.5 &                 0.719205 \\
		KKLE		& 5 & -0.622219 & 0.979745& 0.572215 &  0.9& 4.151828                       \\
		\bottomrule
	\end{tabular}
	
\end{table*}
	
	\subsubsection{Comparison for small data} 
	
	In this section, our goal is to compare the two estimators for a small dataset ($N=100$). We  compare the bias, RMSE, and the variance of the  KKLE estimator with the MINE estimator.  Since the size of the data is small using minibatches did not help. Hence, we use the whole data and run the simulation for 100 gradient steps. In Table 2, we provide the comparisons averaged over 100 trials. We compare the estimators for $D=1$ scenario.  We find that the KKLE estimator has a much lower bias, variance, and  RMSE value. For $D=5$ scenario both the estimators are not reliable for the small dataset setting. Hence, the comparisons in this setting did not provide any insights and were not reported.

	\begin{figure}
	\centering
	\includegraphics[trim= 0 80 0 40, width=7in]{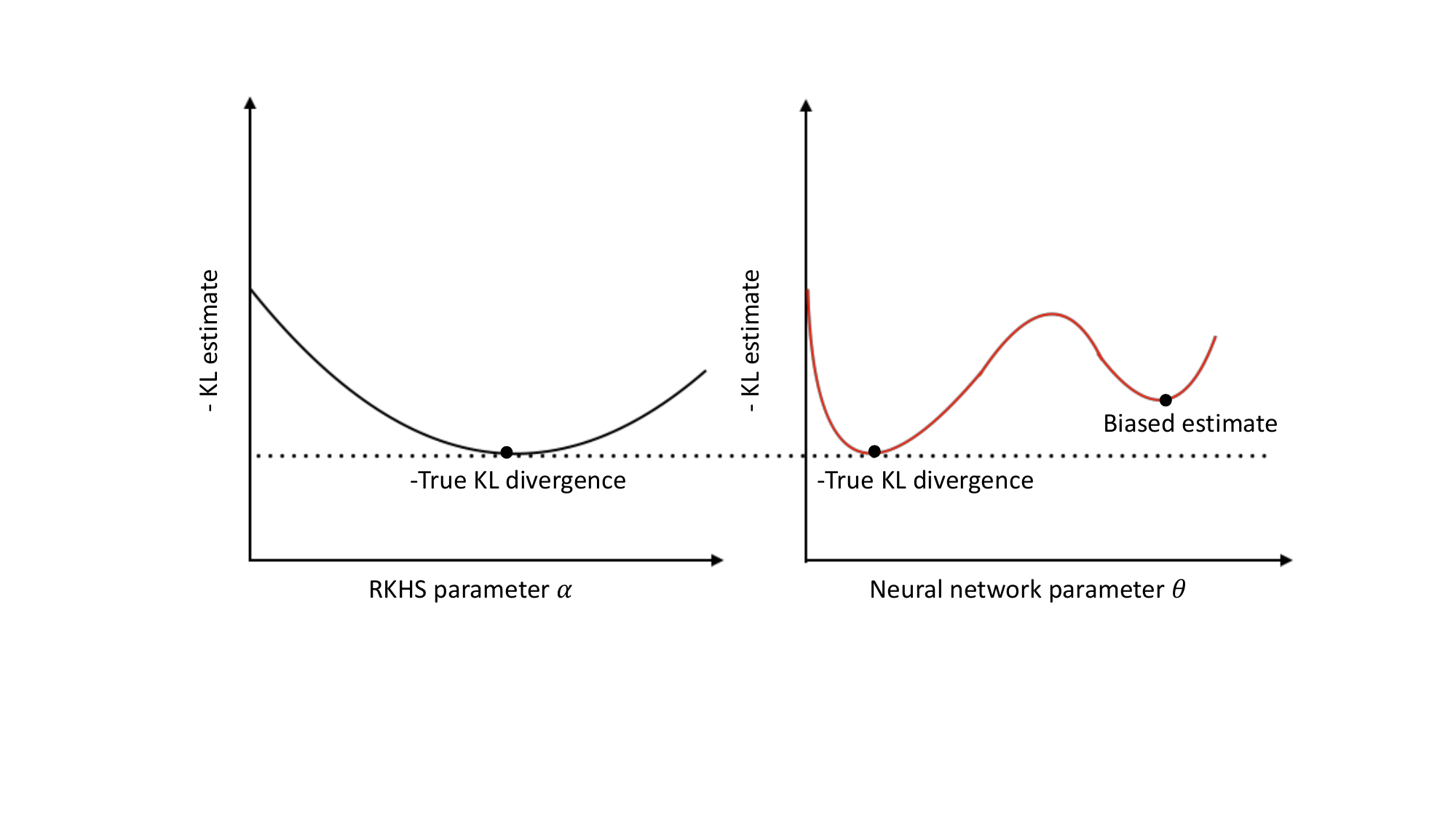}
	\caption{KKLE vs MINE: Comparing the bias }
	\label{fig1}
\end{figure}

	\begin{table*}
	\centering
	\caption{KKLE  vs MINE estimator for small data}
	\begin{tabular}{ccccccc}
		\toprule
		Estimator  &      Bias &      RMSE &      Variance  &  Correlation &  True Mutual Information \\
		\midrule
		MINE	 & 0.09390 &  0.1044 &  0.002091 &       0.2 &                 0.020411 \\
		MINE 	 &0.06810 &  0.1128 &  0.008096    &  0.5 &                 0.143841 \\
		MINE & -0.29106&  0.5123&  0.177734  &   0.9 &                 0.830366 \\
		
		\midrule
		KKLE & 0.04999 &  0.0733&  0.000288 &         0.2 &                 0.020411 \\
		KKLE & 0.06152 &  0.1254 &  0.011950  & 0.5 &                 0.143841 \\
		KKLE & 0.00855 &  0.1833&  0.033570  &0.9 &                 0.830366 \\
		\bottomrule
	\end{tabular}
\end{table*}
	\subsection{Explaining KKLE's performance}
	We conclude that for smaller datasets and smaller dimensions the KKLE estimator performs better than the MINE estimator. 
	When the datasets are very large both MINE and KKLE estimator perform well. 
	\begin{itemize}
		\item The loss surface for MINE is non-convex in the parameters and thus different trials lead to different minimas being achieved thus leading to a higher variance than KKLE,  which searches over a convex loss surface.
		
		\item Hypothetically assume that the search space for KKLE is the same as MINE.
		In such a case, the optimizer for KKLE is likely to have a lower bias and RMSE as it will always find the best minima, which is not true for MINE. We illustrate this in Figure \ref{fig1}.
	\end{itemize}

	\subsection{Application to Metrics for Fairness} There are many applications for mutual information. In this section, we propose another application that can directly benefit from the proposed estimator. Machine learning methods are used in many daily life applications. In many of these applications such as deciding whether to give a loan, hiring decisions, it is very important that the algorithm be fair. There are many definitions of fairness that have been proposed in the literature \cite{speicher2018unified}. 
	We discuss the three most commonly used definitions of fairness here. 
	
	\begin{itemize}
		\item \textbf{Demographic Parity.} A predictor is said to  satisfy demographic parity if the predictor is independent of the protected attribute (for instance, race, gender, etc.).
		\item \textbf{Equality of Odds.} A predictor satisfies equality of odds if the predictor and the protected attribute are independent conditional on the outcomes.
		\item  \textbf{Equality of Opportunity} A predictor satisfies equality of opportunity with respect to a certain class if the predictor and the protected attribute are independent conditional on the class.
	\end{itemize}
	These definitions provide a condition to measure fairness. These conditions serve as a hard constraint and may not be satisfied by any algorithm. Hence, it is important to provide metrics that measure the extent to which these conditions are satisfied. Current works \cite{bellamy2018ai} mainly implement these metrics for fairness when the protected attribute is a categorical variable. Extending these metrics to settings when the protected attribute is continuous (for instance, income level, etc.) is not obvious (See the future works mentioned in \cite{donini2018empirical}). 


	We propose to express these fairness criteria in terms of mutual information. Expressing it in terms of mutual information has two advantages: a) We can understand the extent to which the criterion is satisfied as the new definition won't be a mere hard constraint, and
	b) Dealing with  protected attributes that are continuous (for e.g., income level) becomes more natural.
	
	We give the mathematical formulation next.  Suppose that the predictor random variable is given as $Y^{p}$ (for instance, the prediction that the individual would default on the loan), the ground truth is $Y$ (for instance, if the person actually defaults on the loan), and the protected attribute is given as $A$ (for instance, race, income level etc.).

	\begin{itemize}
		\item \textbf{Demographic Parity} $Y^{p } \perp A \Leftrightarrow I(Y^{p}; A) = 0$
		\item \textbf{Equality of Odds} $Y^{p } \perp A\;|\;Y \Leftrightarrow I(Y^{p}; A\;|\;Y) = 0$
		\item  \textbf{Equality of Opportunity} $Y^{p } \perp A\;|\;Y=1 \Leftrightarrow I(Y^{p}; A\;|\;Y=1) = 0$
	\end{itemize}
	
	Therefore, for each of the above definitions, we require the appropriate value of mutual information to be low. Hence, we can compare the extent of fairness for different machine learning models in terms of the mutual information estimate. In each of the above definitions, we are only required to estimate the mutual information between two random variables, which is good as we know that mutual information estimation is  reliable in lower dimensions. It would be interesting to investigate mutual information based fairness constraints. Further investigation of  mutual information based metrics for fairness in machine learning is an interesting future work.


	\section{Conclusion}
	We propose a new estimator for KL divergence based on kernel machines. We prove that the proposed estimator is consistent. Empirically, we find that the proposed estimator can be more reliable than the existing estimator MINE in different settings. We also provide insights into when KKLE is expected to do better than MINE. 
	
	\section{Acknowledgements}
	
	We would like to thank Prof. Gregory Pottie (University of California, Los Angeles) and Sherjil Ozair (University of Montreal) for valuable discussions and references.
	\bibliographystyle{IEEEtran}
	\bibliography{MIKE_ASILOMAR1}
\end{document}